\PassOptionsToPackage{table}{xcolor}
\documentclass[10pt,twocolumn,letterpaper]{article}

\usepackage{cvpr}              

\definecolor{cvprblue}{rgb}{0.21,0.49,0.74}
\usepackage[pagebackref,breaklinks,colorlinks,allcolors=cvprblue]{hyperref}
\usepackage{optidef}
\usepackage{amsmath, amsthm, amssymb}

\newtheorem{lemma}{Lemma}

\newtheorem{theorem}{Theorem}

\usepackage{verbatim}
\usepackage{adjustbox}
\usepackage{graphicx} 
\usepackage{float}
\usepackage{soul}
\usepackage{multirow}
\usepackage{booktabs}
\usepackage{tocloft}
\usepackage{titletoc}
\usepackage[accsupp]{axessibility}

\def\paperID{17162} 
\def\confName{CVPR}
\def\confYear{2025}

\title{Towards Source-Free Machine Unlearning}

\author{%
Sk Miraj Ahmed\textsuperscript{1,2,}\thanks{Equal contribution.} \textsuperscript{,}\thanks{Work partially done while at University of California, Riverside.},\hspace{0.1cm}
Umit Yigit Basaran\textsuperscript{2,}\footnotemark[1],\hspace{0.1cm}
Dripta S. Raychaudhuri\textsuperscript{3}\textsuperscript{,}\thanks{Work done outside AWS},\hspace{0.1cm}
Arindam Dutta\textsuperscript{2} \\
Rohit Kundu\textsuperscript{2}, Fahim Faisal Niloy\textsuperscript{2}, Ba\c{s}ak G\"{u}ler\textsuperscript{2}, Amit K. Roy-Chowdhury\textsuperscript{2} \\
\textsuperscript{1}Brookhaven National Laboratory \quad
\textsuperscript{2}University of California, Riverside \quad
\textsuperscript{3}AWS AI \\
\texttt{\small \{sahme047@, ubasa001@, drayc001@, adutt020@\}ucr.edu} \\
\texttt{\small \{rkund006@, fnilo001@, bguler@ece., amitrc@ece.\}ucr.edu}
}

\begin{document}

\maketitle
\begin{abstract}
As machine learning becomes more pervasive and data privacy regulations evolve, the ability to remove private or copyrighted information from trained models is becoming an increasingly critical requirement. Existing unlearning methods often rely on the assumption of having access to the entire training dataset during the forgetting process. However, this assumption may not hold true in practical scenarios where the original training data may not be accessible, i.e., the source-free setting. To address this challenge, we focus on the source-free unlearning scenario, where an unlearning algorithm must be capable of removing specific data from a trained model without requiring access to the original training dataset. Building on recent work, we present a method that can estimate the Hessian of the unknown remaining training data, a crucial component required for efficient unlearning. Leveraging this estimation technique, our method enables efficient zero-shot unlearning while providing robust theoretical guarantees on the unlearning performance, while maintaining performance on the remaining data. Extensive experiments over a wide range of datasets verify the efficacy of our method.
\end{abstract}    
\section{Introduction}
\label{sec:intro}

Machine learning models have achieved significant success by training on large amounts of annotated data, much of which may include sensitive or private information \cite{guo2019certified}. With the introduction of data protection rules such as the General Data Protection Regulation (GDPR) \cite{voigt2017eu}, there is a growing need for algorithms that can delete (or forget) information learned from such sensitive datasets. Furthermore, privacy concerns may prompt individuals to request the removal of their data from the training set, invoking their ``right to be forgotten'' \cite{mantelero2013eu}. A straightforward solution to this issue would be to retrain the model from scratch using only the non-private subset of the original dataset. However, retraining is computationally inefficient and impractical (and impossible in the source-free setting we introduce in this paper). This highlights the necessity for efficient \textit{Machine Unlearning} (MU) \cite{golatkar2020eternal, ye2022learning} algorithms that enable modifications to the trained model parameters to forget specified data while maintaining performance on the remaining data.

Although several recent machine unlearning algorithms have demonstrated reasonable success on existing benchmarks \cite{tarun2023fast, kurmanji2024towards}, nearly all current approaches \cite{golatkar2020eternal, ye2022learning, tarun2023fast, kurmanji2024towards} assume the availability of the remaining data, either in full or in part. In practical settings, storing such large volumes of data is challenging due to storage costs and privacy issues. Consequently, these methods fail to address scenarios where the \emph{model owner no longer has access to the original training data}. In these situations, ensuring accurate and efficient unlearning becomes markedly more difficult. Without the original data, it is challenging to verify that the specified information has been entirely removed from the model and that the model's performance on the remaining data remains unaffected. This limitation underscores the pressing need to develop robust unlearning techniques that can function effectively even when the original training dataset is inaccessible, i.e., the source-free setting. 

A recent study has introduced a solution for this challenge, referring to the setting as "zero-shot machine unlearning" \cite{chundawat2023zero}, which works by solely requiring access to the trained model weights and the data to be forgotten, without needing the original training dataset. However, a significant limitation of this technique is its inability to forget random instances encompassing diverse classes; it can only selectively forget particular data classes.
This constraint could hinder its practicality in scenarios where users only want certain instances unlearned, as this method discards all the data pertaining to a user, rather than selectively removing certain examples.
To tackle this limitation, another recent investigation \cite{cha2024learning} attempts zero-shot unlearning at the instance level instead of the class level. This approach enables the removal of requested data without requiring access to the complete training dataset. However, it suffers from scalability issues, as increasing the number of instances results in a significant drop in performance on both test data and the remaining dataset, which is undesirable for effective and reliable machine unlearning. Additionally, these methods fall short in providing robust theoretical assurances regarding their respective performance.

Considering the aforementioned issues, \emph{we propose an unlearning algorithm} that excels in such scenarios, \emph{where the original training data is unavailable, while providing robust theoretical guarantees}. We term this as ``source-free machine unlearning", in analogy to source-free domain adaptation methods \cite{liang2020we}. (We believe this is a better term than zero-shot unlearning.) Inspired by \cite{guo2019certified}, we study the unlearning mechanisms of $\ell_2$ regularized linear models with differentiable convex loss functions. Specifically, \citet{guo2019certified} define a Newton update step on the model parameters, which can be used to perform unlearning. This step is proven to be optimal for the quadratic loss function, and for strongly convex Lipschitz loss functions, the discrepancy between this step and optimal forgetting is bounded. Crucially, this Newton step requires the Hessian of the remaining data with respect to the trained model parameters. However, in our problem setup, we do not have access to the remaining data. Thus, we cannot compute this Hessian directly. 

To address this challenge, we introduce an algorithm that approximately estimates the Hessian of the retained data (referred to as the ``retain Hessian") using only the data designated for removal and the trained model by assuming the loss differences at perturbed points for remaining and forget data are close to each other. By focusing on closely aligning the Hessian estimate with the true value, our approach offers robust theoretical guarantees. This aspect is crucial as it reinforces confidence in the machine unlearning process for data removal, ensuring both accuracy and reliability. 

Our main contributions in this work can be summarized as follows:
\begin{itemize} \item To the best of our knowledge, this work proposes the first zero-shot unlearning method for linear classifiers that can effectively forget arbitrary instances of data across all classes while providing strong theoretical guarantees for data removal and privacy.

\item Recognizing that the Hessian cannot be directly computed from the remaining data, we introduce a novel method for estimating the Hessian based on the data marked for removal. This approach ensures an approximate (i.e., bounded-error) unlearning mechanism applicable to any convex loss functions. While our algorithm is specifically designed for linear classifiers, we demonstrate its applicability in mixed linear scenarios, such as by \textit{linearizing a deep neural network}.

\item We establish \textit{theoretical guarantees} for our unlearning framework through comprehensive proofs and validate our approach with experiments and ablation studies using multiple benchmark datasets.

\end{itemize}
\section{Related Works}
\label{sec:related}

\textbf{Machine unlearning.} Machine unlearning, introduced in \cite{cao2015towards}, aims to efficiently remove the influence of certain training instances from a model's parameters. Unlearning approaches in the literature can be primarily categorized into exact and approximate unlearning methods. Exact unlearning methods aim to ensure that the data is completely unlearned from the model, akin to retraining from scratch. Recent approaches include \cite{bourtoule2021machine, dukler2023safe}, which split the data into multiple shards and train separate models on different non-overlapping combinations of these shards. However, it comes with substantial storage costs since multiple models must be maintained. In contrast, approximate unlearning methods estimate the influence of the unlearning instances and remove it through direct parameter updates. Some approximate methods focus on improving efficiency \cite{wu2020deltagrad} or preserving performance \cite{wu2022puma}, but they lack formal guarantees on data removal. A second group of approximate approaches \cite{guo2019certified,neel2021descent,golatkar2020eternal, golatkar2021mixed} provide theoretical guarantees on the statistical indistinguishability of unlearned and retrained models based on ideas similar to differential privacy \cite{dwork2014algorithmic}. All these methods require access to access to all, or a subset of, the training data. This assumption may not hold true in many practical settings; nevertheless, data privacy concerns may need to be addressed \cite{gdpr2018general}. Recently, machine unlearning has attracted attention and achieved notable success in various applications, such as mitigating bias \cite{chen2024fast}, erasing unwanted or copyrighted content \cite{gandikota2023erasing, li2024machine}, and preventing malicious attacks \cite{yao2023large} in recent large-scale generative models.

\noindent\textbf{Source-free unlearning.} 
A recent paper has proposed a method for unlearning  which works by solely requiring access to the trained model weights and the
data to be forgotten, without needing the original training dataset, referring to it as "zero-shot machine unlearning" \cite{chundawat2023zero}. They propose two approaches: error minimizing-maximizing noise and gated knowledge transfer.
 
The first approach learns a set of noise matrices which maximize the error for the forget set, and a separate set of noise matrices which minimize the error as a proxy for the remaining data. The second approach uses knowledge distillation of the original model into a new model, gated by a filter that prevents the forget set knowledge from being passed, and additionally, supplemented by a generator network for sample generation. A major limitation of these methods is their inability to forget specific instances of different classes; rather, they forget all the data pertaining to a class. 
However, such fine-grained forgetting scenarios are likely to occur in real-world applications, where the need for selective data removal or modification is prevalent. A recent work \cite{cha2024learning} proposes an adversarial sample generation strategy to extend zero-shot unlearning to the instance-wise case. However, this method struggles to scale beyond forgetting a few samples without significantly degrading model performance. Another recent work \cite{foster2024zero} proposes a method called Just-in-Time (JiT) unlearning where they minimize the gradient effect by creating perturbed samples around each forget sample and fine-tune the model with those samples for reducing its reliance on that sample.

Critically, all existing zero-shot machine unlearning methods fail to provide any formal guarantees regarding the completeness or effectiveness of the forgetting process. In practical applications, where data privacy and compliance are paramount, such guarantees are essential to ensure that sensitive information is reliably removed from the model without compromising its overall performance. Additionally, a core contribution of our work is to provide such guarantees when the source data is no longer available.

\section{Preliminaries}
\label{sec:prelim}

The mathematical concept central to the ideal machine unlearning setting is \textit{parameter indistinguishability} \cite{guo2019certified}. In this section, we provide an overview of this definition. Additionally, we present the preliminaries of the unlearning mechanism for linear classifiers under convex losses.

\subsection{Parameter Indistinguishability} 
Consider a data distribution $\mathcal{D} \sim \{x_i, y_i\}_{i=1}^n$ representing a training set used to train a model with a randomized algorithm $\mathcal{A}$ resulting in the output hypothesis space $\mathcal{H}$. Suppose there is a desire to eliminate the influence of $x_i$ from $\mathcal{H}$ using an unlearning mechanism $\Xi$. The unlearning mechanism is said to achieve \textit{parameter indistinguishability}, if $\Xi$ functions in a manner such that the outputs of $\Xi(\mathcal{A}(\mathcal{D}), \mathcal{D}, x_i)$ and $\mathcal{A}(\mathcal{D} \backslash x_i)$ are very close to each other.
The current trend in unlearning research emphasizes demonstrating the efficacy of designed mechanisms using this metric. In simpler terms, the unlearned model should closely mimic, in terms of output space, the model that has been retrained from scratch without the specific data. Further discussion on this aspect will be provided in detail in the experimental section. 
In our case, we explore linear classifiers with the randomized algorithm $\mathcal{A}$ being the supervised learning, using standard convex loss functions.

\subsection{Unlearning of Linear Classifier}
\label{unlearn_linear}
The empirical loss with respect to a linear classifier $w \in \mathbb{R}^d$ and a convex loss function $l: \mathbb{R}^d \rightarrow \mathbb{R}$ can be written as 
$\mathcal{L}(w) =  \sum_{i=1}^{n} l(y_i, w^\top {x}_i) +\frac{\lambda n}{2} \|w\|_2^2$. Let $w^\star = \underset{w}{\text{arg min}} \ \mathcal{L}(w)$ be the optimal linear classifier trained on the distribution $\mathcal{D}$.
To forget a subset of the training data $\mathcal{D}_f \subset \mathcal{D}$, the naive approach involves retraining the classifier over the distribution $\mathcal{D}_r = \mathcal{D}\backslash\mathcal{D}_f$. However, this approach is impractical and time-consuming. A more widely used alternative is to mitigate the influence of the forget dataset using the influence function \cite{guo2019certified, warnecke2021machine} on the optimal model parameters. Mathematically, this unlearning mechanism can be expressed as:
\begin{equation}
    \Xi(w^\star, \mathcal{D},\mathcal{D}_f) = w_{uf} = w^\star + \mathrm{H}_r^{-1} \nabla_f + \sigma^2\varepsilon
\label{unlearn_exact}
\end{equation}


Here, $w_{uf}$ represents the model parameters obtained by unlearning the forget dataset, $w^\star$ is the optimal model parameter obtained using the entire training data, $H_r$ is the Hessian of the remaining dataset, and $\nabla_f$ is the gradient of the forget dataset at the optimal point $w^\star$, $\sigma$ is the variance of the noise term and $\varepsilon$ is the noise sampled from a standard Gaussian distribution. The term $-\mathrm{H}_r^{-1} \nabla_f$ corresponds to the influence of the forget dataset on the model parameters. Also, the noise term seeks to remove any information that could potentially leak due to slight inconsistencies. This unlearning methods is theoretically grounded and the residual norm of the gradient of the unlearned model on the remaining training set $\mathcal{D}_r$ can be tightly upper bounded. However, the assumption of having access to $\mathcal{D}$ during unlearning is strong and we relax this problem where we just have access to $\mathcal{D}_f$. However, without $\mathcal{D}_r$, computing the Hessian $\mathrm{H}_r$ as in \cref{unlearn_exact} is non-trivial. To solve this, we devise a method where we can approximate this $\mathrm{H}_r$ using only $w^\star$ and $\mathcal{D}_f$, which is elaborated in the next section in detail.
\label{param_indis}

\subsection{Mixed Linear Unlearning}
\label{sec:mixed-linear-unlearn}
Due to the highly non-convex structure of the loss landscapes of neural networks, unlearning is a challenging task. As a solution to this problem in \cite{golatkar2021mixed} authors proposed Mixed Linear Unlearning where the loss function is stated as the following:
{\small
\begin{equation*}
    \begin{split}
        \mathcal{L}(w) = \sum_{i = 1}^n \|f_{w_c^\ast}(x_i) + \nabla_wf_{w_c^\ast}(x_i) \cdot w - y_i\|^2_2 
        + \frac{\lambda n }{2}\|w\|^2_2
    \end{split}
\end{equation*}}
where $w_c^\ast$ is the model parameters trained on a public dataset, $f_{w_c^\ast}$ is the function parameterized with $w_c^\ast$ i.e. neural network, and $\nabla_wf_{w_c^\ast}$ is the Jacobian of the function $f_{w_c^\ast}$. The model parameters $w$ are trained on the user-specific dataset where there can be an unlearning request after training. The Neural Tangent Kernel (NTK) \cite{jacot2018neural, li2019enhanced} approach in mixed linear unlearning leverages a first-order Taylor expansion to approximate the network’s behavior in a linearized form, which simplifies the complex task of forgetting in non-linear networks. This approximation transforms the problem into a convex optimization task, allowing data removal to be treated with a closed-form solution rather than requiring full re-optimization. By linearizing the network’s behavior around specific weights, NTKs make the process of forgetting efficient, enabling the model to quickly process data deletion requests without the computational burden of extensive retraining and while preserving most of the model’s predictive performance. Again, the unlearning mechanism proposed for this method is the same as \cref{unlearn_exact}. In our experiments, to show our methods performance on neural networks we utilized the mixed linear unlearning approach. Our proposed method in the next section is easily applicable to this setup.

\section{Methodology}
\label{sec:method}

Given a differentiable convex loss $\mathcal{L}$, we can write the Taylor approximation of this around the optimal classifier $w^\star$ as follows:
%
\begin{equation*}
    \begin{split}
        \mathcal{L}(w)  &\approx \mathcal{L}(w^\star) + \nabla(w^\star)^\top (w - w^\star) \\
        &\quad + \frac{1}{2} (w - w^\star)^\top \textrm{H}(w^\star) (w - w^\star)
    \end{split}
\end{equation*}
%
where the higher-order terms of the Taylor expansion are neglected here due to its relatively small magnitude.
We define the loss difference $ \delta \mathcal{L}$ as follows:
\begin{align}
    \delta \mathcal{L}  &= \mathcal{L}(w) - \mathcal{L}(w^\star) \notag 
\end{align}
Assuming this $\delta \mathcal{L}$ is computed over the whole training data, we denote the loss difference with respect to the retain data $x^r \in \mathcal{D}_r$ as $\delta \mathcal{L}_r$. So,
\begin{equation}
    \begin{split}
        \delta \mathcal{L}_r &\approx \nabla_r(w^\star)^\top (w - w^\star) \\
        &+ \frac{1}{2} (w - w^\star)^\top \textrm{H}_r(w^\star) (w - w^\star)
    \end{split}
    \label{eq:taylor_r}
\end{equation}
Assuming that the training converges to the global optima $w^\star$, we can safely assume that $\nabla(w^\star)=0$, which also means $\nabla_r(w^\star)+ \nabla_f(w^\star)=0 \ \implies \nabla_r(w^\star) = -\nabla_f(w^\star)$. 
Plugging this in \cref{eq:taylor_r} we get the following:

\begin{align}
    \frac{1}{2} (\delta w)^\top \textrm{H}_r(w^\star) (\delta w) - \nabla_f(w^\star)^\top \delta w - \delta \mathcal{L}_r \approx 0\notag 
\end{align}
where, $\delta w = (w - w^\star)$.
With this observation we generate some ($m$ points) small perturbations around the optima $w_i = w^\star + (\delta w)_i$ and calculate the average to formulate the following objective function of the Hessian as follows:
\begin{align*}
    \mathrm{\Psi}(\textrm{H}_r) &= \frac{1}{m} \sum_{i=1}^m \left( \mathrm{f}_i(\textrm{H}_r)\right)^2 
\end{align*}
where
$\mathrm{f}_i(\textrm{H}_r) = \frac{1}{2} (\delta w)_i^\top \textrm{H}_r(w^\star) (\delta w)_i - \nabla_f(w^\star)^\top (\delta w)_i - \delta \mathcal{L}_r(w_i)$.
Since $\mathrm{\Psi}(\textrm{H}_r) \rightarrow 0$ at the optima, minimizing it should output the desired Hessian $\textrm{H}_r$ w.r.t to the retain data. However, since we do not have access to $\mathcal{D}_r$, we can not explicitly compute $\delta \mathcal{L}_r(w_i)$. Instead, we can replace it with next best value which is $\delta \mathcal{L}_f(w_i)$. This is reasonable replacement since $\delta \mathcal{L}(w_i) \leq L\|w_i-w^\star\|$ where $L$ is the Lipschitz constant corresponding to the loss. As a result both $\delta \mathcal{L}_r(w_i)$ and $\delta \mathcal{L}_f(w_i)$ can be upper bounded by $L \|\delta w\| \rightarrow 0 $, for small perturbations. So with the small upper bound we can approximately say that both the quantities are very close to each other.\\
So we define an approximate version of $\mathrm{f}_i$ as follows:
$\Tilde{\mathrm{f}}_i(\textrm{H}_r) = \frac{1}{2} (\delta w)_i^\top \textrm{H}_r(w^\star) (\delta w)_i - \nabla_f(w^\star)^\top (\delta w)_i - \delta \mathcal{L}_f(w_i)$. 
Our final objective becomes:
\begin{align}
    \mathrm{\Tilde{\Psi}}(\textrm{H}_r) &= \frac{1}{m} \sum_{i=1}^m \left( \Tilde{\mathrm{f}}_i(\textrm{H}_r)\right)^2 \notag 
\end{align}
Clearly the $\textrm{H}_r$ is positive semi definite (PSD) for any convex loss functions.
Based on this observation, we formulate the following optimization as a Semi Definite Program (SDP) as follows:
\begin{mini}|l|
{}{\mathrm{\Tilde{\Psi}}(X)}{}{}
\addConstraint { \textrm{X}\succeq 0}
\label{opt:main_opt}
\end{mini}
Since we are approximating the value of $\delta \mathcal{L}_r(w_i)$ instead of using the actual ground truth value, we anticipate that the solution to optimization problem \cref{opt:main_opt} will be approximately close to the true retained Hessian $\textrm{H}_r$. In fact, we can bound the error between the true and estimated Hessian using the following lemma.
\begin{lemma}
    Consider choosing $\delta w \in \mathbb{R}^d$ where each element $\delta w(j)$ of is sampled from $\mathcal{N}(0,1)$. Assuming that the solution of the optimization \cref{opt:main_opt} converges to $\hat{\textrm{H}}_r$, then the frobenius norm of the difference between the Hessian $\textrm{H}_r$ (the actual ground truth Hessian with respect to $\mathcal{D}_r$) and $\hat{\textrm{H}_r}$ can be upper bounded as:
    \[ \|\Delta \textrm{H}_r\|_F \leq \frac{2 \epsilon \sqrt{d}}{(2+d)}\]
where $\epsilon$ is the upper bound on the approximation error of  $\delta \mathcal{L}_r(w_i)$. Mathematically:
    $|\delta \mathcal{L}_r(w_i) - \delta \mathcal{L}_f(w_i)| \leq \epsilon \ \forall i$
\label{lemma1}
\end{lemma}
\begin{proof}
    From the definition of $\Psi(\mathrm{H})$:
\begin{equation}
    \Psi(\mathrm{X}) = \underset{{\delta w \sim \mathcal{N}(\mathbf{0}, \mathbf{I})}}{\mathbb{E}}\left [ (\frac{1}{2}\delta w^\top \mathrm{X} \delta w + \nabla_r^\top \delta w - \delta \mathcal{L}_r)^2\right]
\label{eqn:new-obj}
\end{equation}
By neglecting higher order terms in the taylor approximation we can say, $\delta \mathcal{L}_r \approx \frac{1}{2}\delta w^\top H_r \delta w + \nabla_r^\top \delta w$. Substituting $\delta \mathcal{L}_r$ from Equation \ref{eqn:new-obj}:
\begin{align*}
    \Psi(\mathrm{X}) &=  \mathbb{E}_{\delta w \sim \mathcal{N}(\mathbf{0}, \mathbf{I})}\left [ (\frac{1}{2}\delta w^\top \mathrm{X} \delta w - \frac{1}{2}\delta w^\top \mathrm{H}_r \delta w )^2\right] \\
    &=  \mathbb{E}_{\delta w \sim \mathcal{N}(\mathbf{0}, \mathbf{I})}\left [ (\frac{1}{2}\delta w^\top(\mathrm{X} - \mathrm{H}_r)\delta w)^2\right] \\
    &=  \frac{1}{2}\mathrm{trace}(\mathrm{M}^2) + \frac{1}{4}\mathrm{trace}(\mathrm{M})^2
\end{align*}
where, we define $\mathrm{M} = (\mathrm{X} - \mathrm{H}_r)$.
So, clearly the minimizer of $\Psi(\mathrm{X})$ is at $\mathrm{M} = 0$ or $\mathrm{X} = \mathrm{H}_r$. However it is the ideal case, where we do not approximate $\delta \mathcal{L}_r$. In our algorithm, we are minimizing and approximate objective $\mathrm{\Tilde{\Psi}}(\textrm{X})$. 
Also from the definition we can say $\Tilde{\mathrm{f}}_i(X) = \mathrm{f}_i(\textrm{X})+(\delta \mathcal{L}_r(w_i)-\delta \mathcal{L}_f(w_i))$. Since we assume that $|\delta \mathcal{L}_r(w_i) - \delta \mathcal{L}_f(w_i)| \leq \epsilon \ \forall i$, we can derive the following inequality:
\begin{align*}
    & (\mathrm{f}_i(\textrm{X}) - \epsilon) \leq \Tilde{\mathrm{f}}_i(X) \leq (\mathrm{f}_i(\textrm{X}) + \epsilon) \\
\implies &  \frac{1}{m} \sum_{i=1}^m \left( \mathrm{f}_i(\textrm{X}-\epsilon)\right)^2 \leq \mathrm{\Tilde{\Psi}(X)} \leq \frac{1}{m} \sum_{i=1}^m \left( \mathrm{f}_i(\textrm{X} + \epsilon)\right)^2
\end{align*}
Using this definition and using the derived term for $ \Psi(\mathrm{X})$, we can derive the following:
\begin{align*}
    & \mathrm{\Tilde{\Psi}(X)} \leq \frac{1}{2}\mathrm{trace}(\mathrm{M}^2+ 2\epsilon \mathrm{M}) + \frac{1}{4}\mathrm{trace}(\mathrm{M})^2 \\
    &  \frac{1}{2}\mathrm{trace}(\mathrm{M}^2- 2\epsilon \mathrm{M}) + \frac{1}{4}\mathrm{trace}(\mathrm{M})^2 \leq \mathrm{\Tilde{\Psi}(X)}
\end{align*}
By seperately taking derivatives of the upper and lower bounds above, if we set it to $0$, we get the following bound on the minimizer $\mathrm{M}$ (\textbf{details in the supplementary)}.
\begin{equation*}
    -\frac{2\epsilon}{(2+d)}\mathrm{I}_d \leq \mathrm{M} \leq \frac{2\epsilon}{(2+d)}\mathrm{I}_d
\end{equation*}
where, $\mathrm{I}_d \in \mathbb{R}^{d \times d}$ is the identity matrix. This inequality implies that the if the solution of optimization~\ref{opt:main_opt} is $\hat{\textrm{H}_r}$, then 
\begin{equation*}
    \textrm{H}_r-\frac{2\epsilon}{(2+d)}\mathrm{I}_d \leq \hat{\textrm{H}}_r \leq \textrm{H}_r+\frac{2\epsilon}{(2+d)}\mathrm{I}_d
\end{equation*}
As a result we can conclude:
\begin{equation*}
   \|\hat{\textrm{H}}_r - \textrm{H}_r\| = \| \Delta \textrm{H}_r\|_F \leq \frac{2 \epsilon \|\mathrm{I}_d\|_F}{(2+d)}
\end{equation*}
Since $\|\mathrm{I}_d\|_F = \sqrt{d}$, we conclude the proof.
\end{proof}
\noindent \textbf{Implications of \cref{lemma1}:} The lemma provides an upper bound on the Frobenius norm between the true and estimated Hessians, characterized by two parameters: $\epsilon$ and $d$. When $\epsilon$ is smaller, the upper bound decreases, leading to a better approximation of the retain Hessian. This is intuitive, as $\epsilon$ represents the approximation error of the loss difference. Additionally, increasing the feature dimension $d$ causes the upper bound to approach zero. Specifically, the norm decreases inversely with $\sqrt{d}$.
In other words, as the matrix size grows, the upper bound on the difference becomes smaller. Importantly, we do not make any assumptions about the linearity of the model in proving this lemma. The bound suggests that for high-dimensional scenarios (such as in deep models), estimating the Hessian accurately may be feasible, making this method worth exploring. However, storing and inverting a large Hessian is computationally impractical. Fortunately, various methods exist for Hessian approximation, such as diagonalization \cite{yao2021adahessian} or linearizing deep models using Hessian-vector products \cite{golatkar2021mixed}. Integrating these approximation techniques with our approach could open up promising avenues for future research.
\begin{theorem}
\label{theorem1}
Suppose that $\forall (x_i, y_i) \in \mathcal{D}$, $w \in \mathbb{R}^d$: $\|\nabla \ell (w^\top x_i, y_i)\| \leq C$. Suppose that the second derivative of $\ell$ is $\gamma$-lipschitz and $\|x_i\|_2 \leq 1$ for all $(x_i, y_i) \in \mathcal{D}$, and the result of optimization \cref{opt:main_opt} is $\hat{H}$. Then:
\begin{align*}
\|\nabla \mathcal{L}(w_{uf}, \mathcal{D}_r)\|_2 &\leq \gamma (n - n_f) \|{\hat{\textrm{H}}_r}^{-1} \nabla_f\|_2^2 \notag \\
&\leq \frac{4 \gamma C^2 n_f^2 (n - n_f)}{\left[\lambda(n - n_f) - \frac{2 \epsilon}{(2+d)}\right]^2}
\end{align*}
where $n$ and $n_f$ denote the number of total training samples and the number of samples to be forgotten respectively.
\end{theorem}
\begin{proof}
    This proof is inspired by \cite{guo2019certified}, and based on \textbf{Theorem 4} of the paper. This bound says that upon forgetting $n_f$ samples from the dataset, if the resulting model become $w_{uf}$ then the norm of the gradient with respect to this model on the remaining dataset can be upper bounded as follows:
    \[
    \|\nabla \mathcal{L}(w_{uf}, \mathcal{D}_r)\|_2 \leq \gamma (n-n_f) \|\mathrm{H}_r^{-1} \nabla_f\|_2^2
    \]
Now this $\mathrm{H}$ is the actual retain hessian while our estimate is $\hat{\textrm{H}}_r = \textrm{H}_r+ \Delta \textrm{H}_r$.
From the definition of loss $\mathcal{L}$, we know that after removing $n_f$ samples the loss becomes $\lambda(n-n_f)$-strongly convex. As a result we get $\|\textrm{H}_r\|_2 \geq \lambda(n-n_f)$. Now we can apply triangle inequality and the upper bound from Lemma \ref{lemma1} as follows:
\begin{align*}
\lambda(n - n_f) &\leq \|\textrm{H}_r\|_2 \leq \|\hat{\textrm{H}}_r\|_2 + \|\Delta \textrm{H}_r\|_2 \notag \\
&\leq \|\hat{\textrm{H}}_r\|_2 + 
\left\|\frac{2\epsilon}{(2+d)} \mathrm{I}_d\right\|_2 \\
& = \|\hat{\textrm{H}}_r\|_2 + \frac{2\epsilon}{(2+d)}
\end{align*}
which implies,
\begin{align*}
& \|\hat{\textrm{H}}_r\|_2 \geq \lambda(n - n_f) - \frac{2\epsilon}{(2+d)} \notag \\
\implies &\|\hat{\textrm{H}}_r\|_2^{-1} \leq \frac{1}{\lambda(n - n_f) - \frac{2\epsilon}{(2+d)}}
\end{align*}
Also, from \textbf{Theorem 4} of \cite{guo2019certified}, we know $\|\nabla_f\| \leq 2C n_f$.
So,
\[
\|\hat{\textrm{H}}_r^{-1} \nabla_f\|_2^2 \leq \|\hat{\textrm{H}}_r^{-1}\|_2^2 \|\nabla_f\|_2^2 \leq \frac{4C^2 n_f^2}{(\lambda(n-n_f) - \frac{2\epsilon}{(2+d)})^2}
\]
\[
\implies \|\nabla \mathcal{L}(w_{uf}, \mathcal{D}_r)\|_2 \leq \frac{4\gamma C^2 n_f^2(n-n_f)}{(\lambda(n-n_f) - \frac{2\epsilon}{(2+d)})^2}
\]
Hence we conclude the proof of Theorem \ref{theorem1}.
\end{proof}
\noindent \textbf{Implications of \cref{theorem1}:} The leftmost term in the theorem's inequality essentially represents the norm of the gradient with respect to the unlearned model on the remaining data. Successful unlearning, as suggested by the parameter indistinguishability (see \cref{param_indis}), should lead this norm to approach \(0\). Examining the upper bound, we observe that for a fixed size \(d\) of the Hessian, it becomes tighter with an increase in the number of remaining data, as the term is inversely proportional to the remaining data size. We validate this phenomenon through experimentation, where we observe a decline in unlearning performance as the number of samples to be forgotten increases. Additionally, the upper bound decreases as we significantly increase the Hessian dimension \(d\). This finding aligns with \cref{lemma1}, which asserts that for large \(d\), the estimated and true Hessians closely resemble each other, indicating effective unlearning.

\section{Experiments}
\label{sec:experiment}

\noindent \textbf{Datasets.} To demonstrate the efficacy of our proposed algorithms for the source-free unlearning scenario, we use four standard benchmark classification datasets: CIFAR-10 \cite{krizhevsky2009learning}, CIFAR-100 \cite{krizhevsky2009learning}, Stanford Dogs \cite{KhoslaYaoJayadevaprakashFeiFei_FGVC2011}, and CalTech-256 \cite{griffin2007caltech}. CIFAR-10 is a dataset consisting of 60,000 RGB images in 10 different classes, with 6,000 images per class. CIFAR-100 is similar to CIFAR-10, but with 100 classes containing 600 images each, providing a more granular classification challenge. Stanford Dogs contains 20,580 images of 120 different breeds of dogs, making it ideal for fine-grained classification tasks. CalTech-256 comprises 30,607 images across 256 object categories, offering a diverse set of images for comprehensive object recognition research. \label{sec:datasets} 

\noindent \textbf{Implementation details.} Since we perform zero-shot unlearning for linear classifiers, we use a ResNet-18 \cite{he2016deep} architecture pre-trained on ImageNet \cite{deng2009imagenet} as our feature extractor. Using these features, we train a linear classifier and then discard the data. During unlearning, we only use the trained linear model and the data to be forgotten. For the experiments with neural networks we used Mixed-Linear neural network \cite{golatkar2021mixed} and we linearized the last few layers. We used ResNet-18 \cite{he2016deep} and the datasets given above for these experiments. We randomly sample up to \(10\%\) of the training data as the forget data. All experiments were performed on a single NVIDIA-RTX 3090 GPU\@. The implementation can be found in \url{https://github.com/UCR-Vision-and-Learning-Group/mixed-linear-forgetting}.

\noindent \textbf{Baseline metrics.} The main baseline for any Machine Unlearning (MU) methods is the parameter indistinguishability between the retrained and unlearned models. A successful unlearning algorithm should emulate the performance of a model that was never exposed to the forget data, having been trained solely on the remaining data. In this context, we evaluate the classification accuracy of the model on the following datasets: (i) test data, (ii) remaining training data, and (iii) forget data. Additionally, we assess the Membership Inference Attack (MIA) score of the models, as proposed in the prior work \cite{kurmanji2024towards}. This score indicates whether a sample was originally part of the training set. After forgetting certain samples, we check their MIA scores using the unlearned model. An MIA score close to 50\% signifies successful unlearning, as it indicates that the unlearned model cannot distinguish whether the forget data came from the training distribution or the test distribution. 

\noindent \textbf{Baseline models.} Unlearned models using our proposed algorithm are compared with two types of models: (a) A model retrained from scratch using the remaining training data (Retrained) and (b) An unlearned model that has been unlearned using the exact Hessian computed from the remaining data (Unlearned(+)). Since we estimate the retain Hessian without accessing the remaining data, our primary objective is to closely mimic the performance of the model described in (a) (marked with cyan color for all the tables) using the baseline metrics. Since we do not need the training data during unlearning we refer the unlearned model using \textit{our proposed algorithm as \textbf{(Unlearned(-))}}. We explore these model's performances using both linear and mixed-linear classifiers for all the datasets.

\subsection{Comparison of baseline metrics on different datasets using linear classifiers}
\begin{figure*}[htbp]
  \centering
  \subfloat[CIFAR-10\label{fig:image1}]{
    \includegraphics[width=0.38\linewidth]{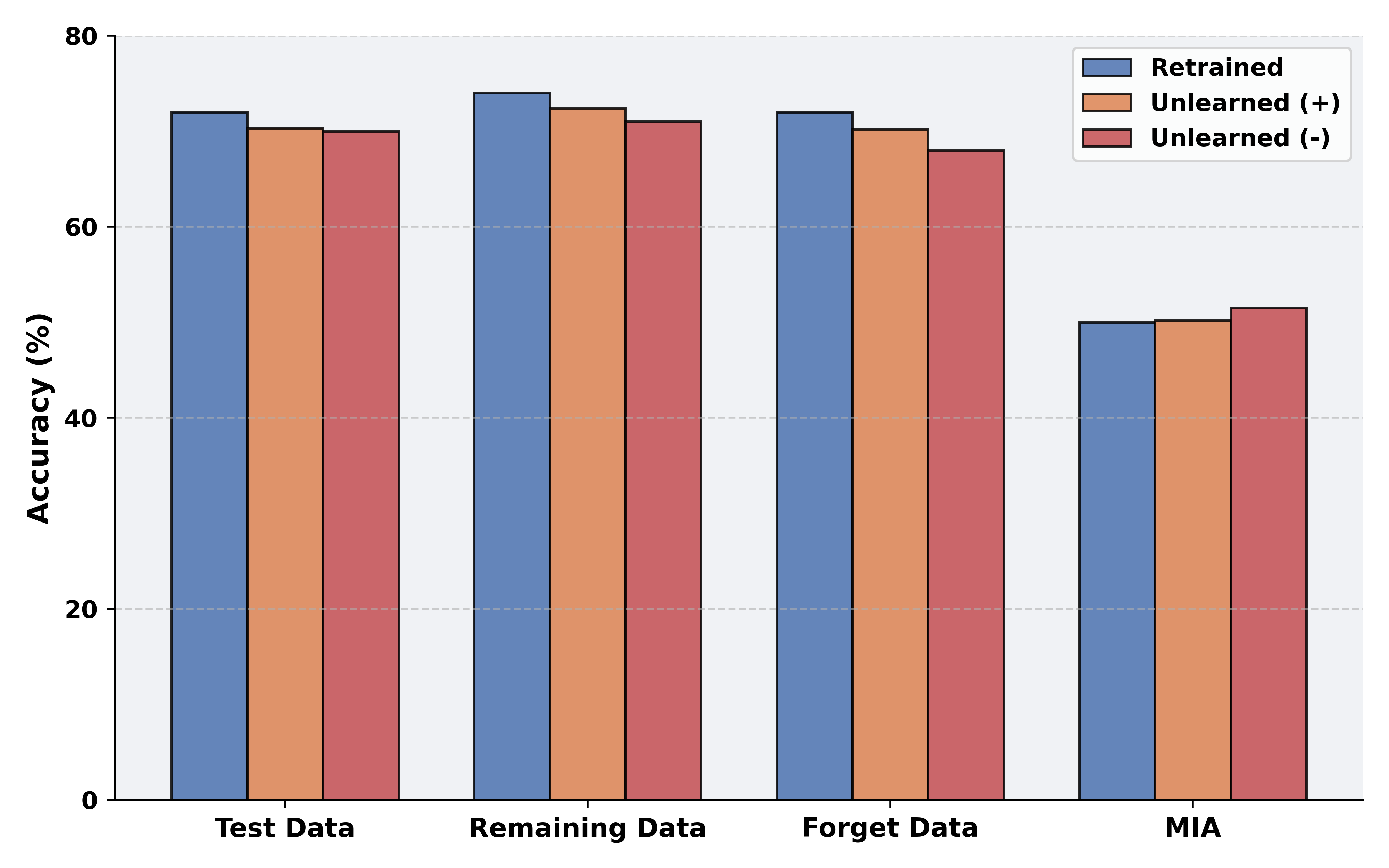}
  }
  \hspace{0.06\linewidth}
  \subfloat[CIFAR-100\label{fig:image2}]{
    \includegraphics[width=0.38\linewidth]{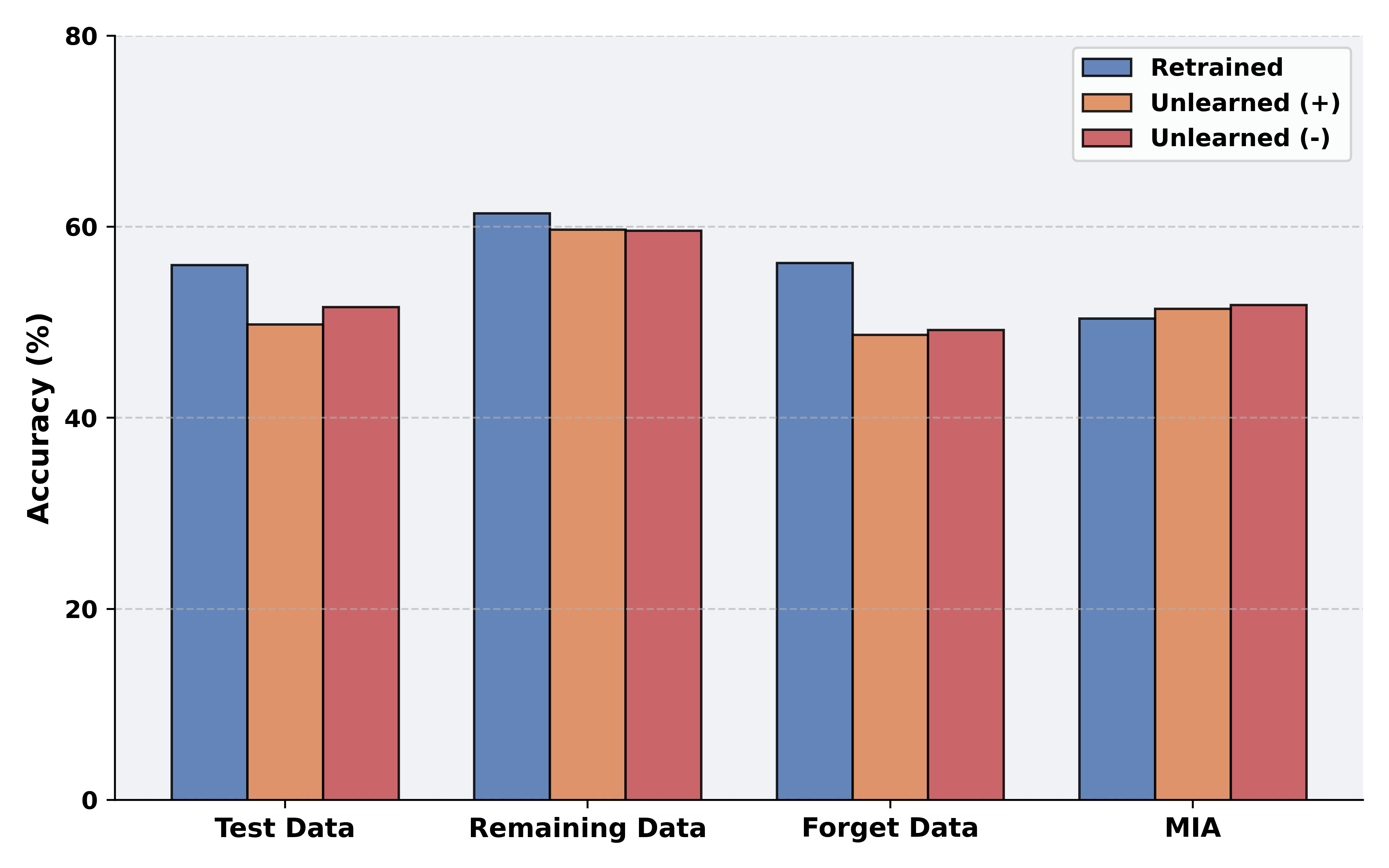}
  }
  \hspace{0.06\linewidth}
  \subfloat[StanfordDogs\label{fig:image3}]{
    \includegraphics[width=0.38\linewidth]{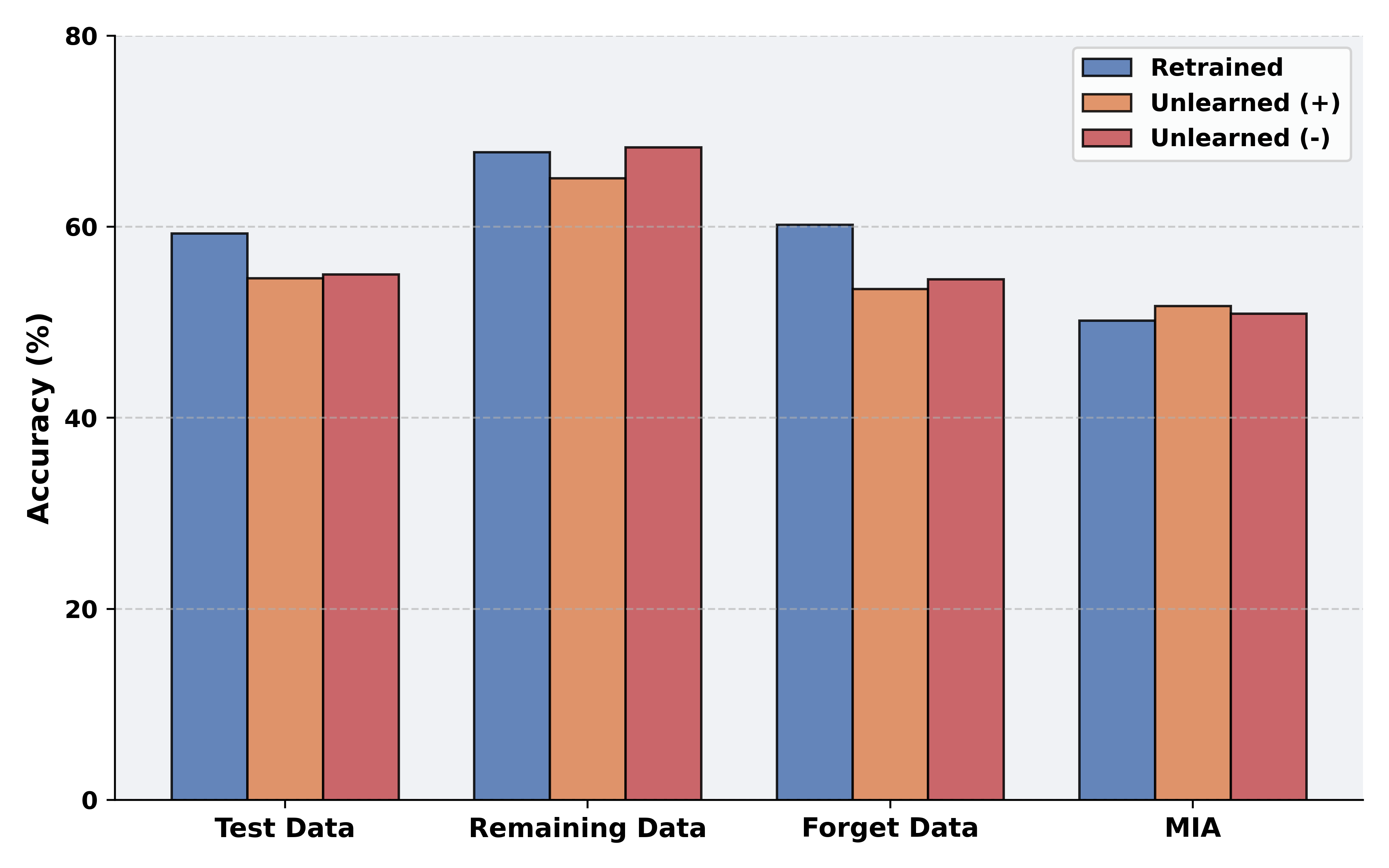}
  }
  \hspace{0.06\linewidth}
  \subfloat[Caltech256\label{fig:image4}]{
    \includegraphics[width=0.38\linewidth]{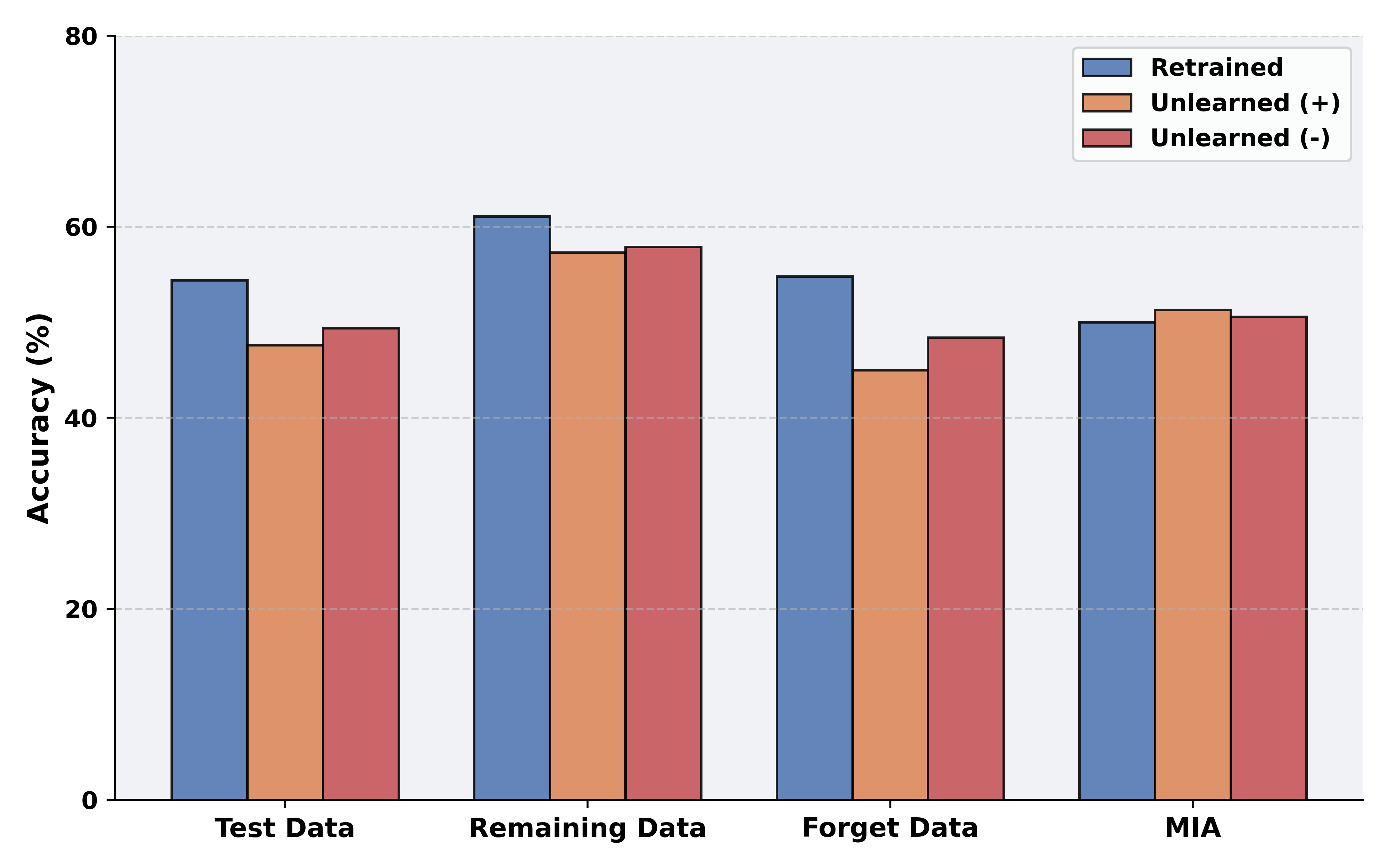}
  }
  \caption{Performance comparison of the proposed methods across different datasets: CIFAR-10, CIFAR-100, StanfordDogs, and Caltech256. We randomly select 10\% of the entire training data as forget samples. Each figure illustrates the effectiveness of the optimization strategies in handling the forgetting of samples, as evidenced by the close performance of models \textit{Unlearned(+)} and \textit{Unlearned(-)}. As expected, our method also yields results comparable to retraining from scratch (\textit{Retrained}), suggesting successful unlearning.}
  \label{fig:method-comparison}
\end{figure*}
We compare the performance of the Retrained, Unlearned(+), and Unlearned(-) models on all four datasets (\cref{fig:method-comparison}) selecting 10\% of the training samples as forget data. We use linear classifier and quadratic loss as the convex loss function for all cases. The results are presented as bar plots for all these scenarios. As theoretically expected, the performance of Unlearned(+) closely mimics that of the Retrained model. As per the main results, in all cases, the performance of Unlearned(-) closely matches that of the Unlearned(+) model, which aligns with our theoretical bounds.
\subsection{Effects of percentage of the forget data size}
To investigate the influence of forget data size, we vary the proportion of randomly selected data for forgetting within the training set while maintaining consistency across all other factors. According to \cref{theorem1}, it becomes apparent that the quantity of forgotten samples significantly influences the optimization process. As we can see in \cref{tab:cifar10-effect-of-forget-data-size}, increasing the number of forget samples negatively impacts performance. Specifically, when 5\% of the training data is chosen randomly for forgetting, the disparity between the retrained model with the remaining data and the model updated using our approach becomes negligible. However, with an increase in the percentage of forget data, the gap between these two models widens considerably. This result perfectly matches the bound we provide in \cref{theorem1}. 

\captionsetup[table]{skip=10pt}
\begin{table}[ht]
  \centering
  \renewcommand{\arraystretch}{0.9} 
  \scriptsize 
  \setlength{\tabcolsep}{6pt} 
  \rowcolors{3}{gray!25}{white}
  \begin{tabular}{l c c c c c}
    \toprule
    \rowcolor{gray!50}
    \textbf{Method} && Test & Remaining & Forget & MIA \\
    \midrule
    \rowcolor{cyan!15}
    Retrained && 73\% & 75\% & 73\% & 50\% \\
    Unlearned (-) & 15\% & 59\% & 60\% & 59\% & 55.8\% \\
    Performance Gap && \textbf{14\%} & \textbf{15\%} & \textbf{14\%} & \textbf{5.8\%} \\
    \midrule
    \rowcolor{cyan!15}
    Retrained && 72\% & 74\% & 72\% & 50\% \\
    Unlearned (-) & 10\% & 70\% & 71\% & 68\% & 51.4\% \\
    Performance Gap && \textbf{2\%} & \textbf{3\%} & \textbf{4\%} & \textbf{1.4\%} \\
    \midrule
    \rowcolor{cyan!15}
    Retrained && 73\% & 74\% & 73\% & 49.4\% \\
    Unlearned (-) & 5\% & 73\% & 74\% & 73\% & 49.4\% \\
    Performance Gap && \textbf{0\%} & \textbf{0\%} & \textbf{0\%} & \textbf{0\%} \\
    \bottomrule
  \end{tabular}
  \caption{The effect of the proportion of randomly selected data from the CIFAR-10 training dataset for forgetting. As the number of forget data samples increases, the difference in performance between the Retrained and Unlearned(-) also increases. The second column indicates the percentage of the selected forgetting data.}
  \label{tab:cifar10-effect-of-forget-data-size}
\end{table} 

\subsection{Effects of the number of perturbations}
For this experiment, we perform unlearning with linear classifiers by varying the number of perturbations. As demonstrated in the proof of \cref{lemma1}, the convergence of our optimization toward the true retain Hessian is strongly influenced by the objective function, which is formulated as the expectation of a perturbed random variable. Consequently, to better capture the expected value, a larger number of perturbations is required, which, in turn, positively impacts unlearning performance. This effect is clearly illustrated in Table \ref{tab:cifar10-effect-of-perturb}, where we observe that a higher number of perturbations consistently leads to improved outcomes.

\begin{table}[ht]
  \centering
  \renewcommand{\arraystretch}{0.9} 
  \scriptsize 
  \setlength{\tabcolsep}{6pt} 
  \rowcolors{3}{gray!25}{white}
  \begin{tabular}{l c c c c c}
    \toprule
    \rowcolor{gray!50}
    \textbf{Method} && Test & Remaining & Forget & MIA \\
    \midrule
    \rowcolor{cyan!15}
    Retrained && 72\% & 74\% & 72\% & 50\% \\
    \midrule
    Unlearned (-) & 250 & 57\% & 58\% & 57\% & 56.2\% \\
    Performance Gap && \textbf{15\%} & \textbf{16\%} & \textbf{15\%} & \textbf{6.2\%} \\
    \midrule
    Unlearned (-) & 500 & 70\% & 71\% & 68\% & 51.4\% \\
    Performance Gap && \textbf{2\%} & \textbf{3\%} & \textbf{4\%} & \textbf{1.4\%} \\
    \midrule
    Unlearned (-) & 1000 & 72\% & 74\% & 71\% & 49\% \\
    Performance Gap && \textbf{0\%} & \textbf{0\%} & \textbf{1\%} & \textbf{1\%} \\
    \bottomrule
  \end{tabular}
  \caption{The effect of the number of perturbations randomly selected from a Gaussian distribution for the CIFAR-10 dataset. The second column indicates the number of perturbations used to approximate the Hessian using our method. Increasing the number of perturbations positively influences unlearning performance.}
  \label{tab:cifar10-effect-of-perturb}
\end{table}

\subsection{Effects of the L2 regularization}
\begin{table}[ht]
  \centering
  \renewcommand{\arraystretch}{0.9} 
  \scriptsize 
  \setlength{\tabcolsep}{6pt} 
  \rowcolors{3}{gray!25}{white}
  \begin{tabular}{l c c c c c}
    \toprule
    \rowcolor{gray!50}
    \textbf{Method} && Test & Remaining & Forget & MIA \\
    \midrule
    \rowcolor{cyan!15}
    Retrained && 72\% & 74\% & 72\% & 50\% \\
    \midrule
    Unlearned (-) & 0 & 70\% & 71\% & 68\% & 51.4\% \\
    Performance Gap && \textbf{2\%} & \textbf{3\%} & \textbf{4\%} & \textbf{1.4\%} \\
    \midrule
    Unlearned (-) & 0.0005 & 71\% & 72\% & 71\% & 49.8\% \\
    Performance Gap && \textbf{1\%} & \textbf{2\%} & \textbf{1\%} & \textbf{0.2\%} \\
    \midrule
    Unlearned (-) & 0.001 & 72\% & 73\% & 72\% & 50.9\% \\
    Performance Gap && \textbf{0\%} & \textbf{1\%} & \textbf{0\%} & \textbf{0.9\%} \\
    \bottomrule
  \end{tabular}
  \caption{The impact of the regularization parameter $\lambda$ on the unlearning algorithm. Increasing $\lambda$ leads to improved unlearning performance, consistent with our claim in Theorem~\ref{theorem1}.}
  \label{tab:cifar10-effect-of-l2}
\end{table}

\noindent The theoretical upper bound on the norm in \cref{theorem1} is clearly proportional to $\frac{1}{\lambda^2}$, with $\lambda$ representing the regularization parameter. Consequently, as demonstrated in \cref{tab:cifar10-effect-of-l2}, increasing $\lambda$ leads to a reduction in the performance gap between our unlearned model and the retrained model.
\subsection{Experiments on mixed linear networks}
To show the applicability of our approach to the neural networks, we consider using mixed linear networks for unlearning \cite{golatkar2021mixed}. We select the last few layers of ResNet-18 to linearize and train it with CIFAR-100. We randomly select 10\%, 15\%, and 20\% of data to be forgotten from the trained model and apply our optimization function to approximate the remaining data Hessian at the trained model. The results are listed in \cref{tab:mixed-linear-exp}. As can be seen, our method can approximate the exact remaining hessian and can perform significantly well even without having the remaining data. We also conduct experiments on other datasets mentioned in \cref{sec:datasets} [Datasets](See supplementary).

\captionsetup[table]{skip=10pt}
\begin{table}[ht]
  \centering
  \renewcommand{\arraystretch}{0.9} 
  \scriptsize 
  \setlength{\tabcolsep}{6pt} 
  \rowcolors{3}{gray!25}{white}
  \begin{tabular}{l c c c c c}
    \toprule
    \rowcolor{gray!50}
    \textbf{Method} && Test & Remaining & Forget & MIA \\
    \midrule
    \rowcolor{cyan!15}
    Retrained && 62.2\% & 65.7\% & 62.6\% & 50\% \\
    Unlearned (+) & 20\% & 60.1\% & 63.4\% & 59.3\% & 50.4\% \\
    Unlearned (-) && 60\% & 70\% & 63.2\% & 50.8\% \\
    Performance Gap && \textbf{0.1\%} & \textbf{6.6\%} & \textbf{3.9\%} & \textbf{0.4\%} \\
    \midrule
    \rowcolor{cyan!15}
    Retrained && 63.1\% & 68.9\% & 63.3\% & 50.2\% \\
    Unlearned (+) & 15\% & 61.9\% & 67.9\% & 62.2\% & 50.7\% \\
    Unlearned (-) && 61.4\% & 70.1\% & 62.2\% & 51.7\% \\
    Performance Gap && \textbf{0.5\%} & \textbf{2.2\%} & \textbf{0.0\%} & \textbf{1.0\%} \\
    \midrule
    \rowcolor{cyan!15}
    Retrained && 63.7\% & 72\% & 63.9\% & 50\% \\
    Unlearned (+) & 10\% & 63.3\% & 72.5\% & 64.9\% & 51\% \\
    Unlearned (-) && 62.8\% & 70.1\% & 61.2\% & 52.3\% \\
    Performance Gap && \textbf{0.5\%} & \textbf{2.4\%} & \textbf{3.7\%} & \textbf{1.3\%} \\
    \bottomrule
  \end{tabular} 
  \caption{We demonstrate the applicability of our method to neural networks using mixed linear networks. The method performs considerably well even without access to remaining data. Experiments were conducted on the CIFAR-100 dataset with a ResNet-18 model, linearizing the last few layers.}
  \label{tab:mixed-linear-exp}
\end{table}

\subsection{Comparison with other source-free unlearning approaches}

We compared our method with other unlearning approaches closely related to our work. For these experiments, we utilized the last layer of the ResNet-18 architecture as a linear classifier on the CIFAR-10 dataset, randomly selecting 10\% of the data to be forgotten across all tests. As shown in \cref{tab:source-free}, our approach outperforms existing source-free methods significantly in terms of performance on remaining, forget, and test data. Specifically, \textbf{NegGrad} \cite{golatkar2020eternal} fine-tunes using only the forget data by applying gradient descent to increase the loss for forget data. The \textbf{Random Labels} \cite{golatkar2020eternal} method randomly reassigns labels (excluding the actual class) to forget samples and fine-tunes the model with this modified data. \textbf{JiT} \cite{foster2024zero} employs Lipschitz regularization to maintain stable model outputs across perturbed data samples for unlearning. The \textbf{Adversarial} \cite{cha2024learning} approach combines adversarial examples with weight importance metrics to preserve performance on remaining data while applying gradient ascent on the forget data. 


\begin{table}[H]
  \renewcommand{\arraystretch}{1} 
  \small 
  \resizebox{\columnwidth}{!}{ 
    \rowcolors{3}{gray!25}{white}
    \begin{tabular}{l c c c c}
      \toprule
      \rowcolor{gray!50}
      \textbf{Method} & \textbf{Test Data} & \textbf{Remaining Data} & \textbf{Forget Data} & \textbf{MIA} \\
      \midrule
      \rowcolor{cyan!15} 
      Retrained & 72\% & 74\% & 72\% & 50\% \\
      NegGrad & 51.9\% & 53.2\% & 51.2\% & 48\% \\
      Random Labels & 20.6\% & 21.6\% & 21.4\% & 47\% \\
      JiT & 52.1\% & 53.1\% & 51.1\% & 49.1\% \\
      Adversarial & 51.5\% & 52.7\% & 51.0\% & 50.0\% \\
      Unlearned (-) & \textbf{70\%} & \textbf{71\%} & \textbf{68\%} & \textbf{51.4\%} \\
      \bottomrule
    \end{tabular}
  }
  \caption{Comparison of existing source-free unlearning methods with our proposed method (Unlearned (-)). Our method significantly outperforms others. Experiments are conducted on the CIFAR-10 dataset using a linear classifier.}
  \label{tab:source-free}
\end{table}

\section{Conclusion}
\label{sec:conclusion}


In this paper, we introduce and evaluate a novel unlearning algorithm tailored for linear and mixed linear classifiers, specifically targeting scenarios where the original training data is unavailable during the unlearning process. Our algorithm is a general-purpose method adaptable to a wide range of convex loss functions, enabling its application across diverse contexts where different convex loss functions are employed. This flexibility makes it a versatile tool for unlearning in various machine learning applications. 
We establish robust theoretical bounds for our algorithm, providing assurances of its reliability and effectiveness in unlearning tasks. These theoretical guarantees offer a solid basis for understanding the algorithm’s behavior and performance. Additionally, we examine the implications of these bounds and validate the practical efficacy of our algorithm through extensive experimental evaluations. The results demonstrate that the proposed algorithm performs exceptionally well, confirming its theoretical advantages and highlighting its potential for real-world applications.
\section{Acknowledgment}
\label{sec:acknowledgment}

This work was supported in part by the NSF CAREER Award CCF-2144927, NSF Award CCF-2008020, DURIP N000141812252, the UCR OASIS Fellowship, and the Amazon Research Award.

{
    \small
    \bibliographystyle{ieeenat_fullname}
    \bibliography{main}
}

\clearpage
\setcounter{page}{1}
\maketitlesupplementary

\onecolumn

 \section{Proof for Lemma 1 in more details}

 \begin{proof}
    From the definition of $\Psi(\mathrm{H})$:
\begin{equation}
    \Psi(\mathrm{X}) = \mathbb{E}_{\delta w \sim \mathcal{N}(\mathbf{0}, \mathbf{I})}\left [ (\frac{1}{2}\delta w^\top \mathrm{X} \delta w + \nabla_r^\top \delta w - \delta \mathcal{L}_r)^2\right]
    \label{eqn:new-obj-supp}
\end{equation}

By neglecting higher order terms in the taylor approximation we can say, $\delta \mathcal{L}_r \approx \frac{1}{2}\delta w^\top H_r \delta w + \nabla_r^\top \delta w$. Substituting $\delta \mathcal{L}_r$ from Equation \ref{eqn:new-obj-supp}:
\begin{align}
    \Psi(\mathrm{X}) &=  \mathbb{E}_{\delta w \sim \mathcal{N}(\mathbf{0}, \mathbf{I})}\left [ (\frac{1}{2}\delta w^\top \mathrm{X} \delta w - \frac{1}{2}\delta w^\top \mathrm{H}_r \delta w )^2\right] \\
    &=  \mathbb{E}_{\delta w \sim \mathcal{N}(\mathbf{0}, \mathbf{I})}\left [ (\frac{1}{2}\delta w^\top(\mathrm{X} - \mathrm{H}_r)\delta w)^2\right] \\
    &= \mathbb{E}_{\delta w \sim \mathcal{N}(\mathbf{0}, \mathbf{I})}\left [ (\frac{1}{2}\delta w^\top \mathrm{M} \delta w)^2\right]
\end{align}
where, we define $\mathrm{M} = (\mathrm{X} - \mathrm{H}_r)$.
We will now prove the following:

\begin{equation}
    \mathbb{E}_{\delta w \sim \mathcal{N}(\mathbf{0}, \mathbf{I})}\left [ (\frac{1}{2}\delta w^\top \mathrm{M} \delta w)^2\right] = \frac{1}{2}\mathrm{trace}(\mathrm{M}^2) + \frac{1}{4}\mathrm{trace}(\mathrm{M})^2
\end{equation}

\section*{Proof of Expectation}

We aim to prove the following equation:
\begin{equation}
\mathbb{E}_{\delta w \sim \mathcal{N}(\mathbf{0}, \mathbf{I})} \left[\left(\frac{1}{2} \delta w^\top \mathrm{M} \delta w\right)^2\right] = \frac{1}{2} \mathrm{trace}(\mathrm{M}^2) + \frac{1}{4} \mathrm{trace}(\mathrm{M})^2,
\end{equation}
where \( \delta w \sim \mathcal{N}(\mathbf{0}, \mathbf{I}) \) is a Gaussian random vector with zero mean and identity covariance, and \( \mathrm{M} \) is a symmetric matrix.

\subsection*{Step 1: Reformulation of the Expectation}

Let \( X = \delta w^\top \mathrm{M} \delta w \). Then, the left-hand side can be expressed as:
\begin{equation}
\mathbb{E}\left[\left(\frac{1}{2} \delta w^\top \mathrm{M} \delta w\right)^2\right] = \frac{1}{4} \mathbb{E}[X^2],
\end{equation}
where \( X^2 = (\delta w^\top \mathrm{M} \delta w)^2 \). Substituting \( X = \delta w^\top \mathrm{M} \delta w \), we expand \( X^2 \):
\begin{equation}
X^2 = (\delta w^\top \mathrm{M} \delta w)^2 = \sum_{i,j} \sum_{k,l} \mathrm{M}_{ij} \mathrm{M}_{kl} \delta w_i \delta w_j \delta w_k \delta w_l,
\end{equation}
where \( \mathrm{M}_{ij} \) denotes the \( (i,j) \)-th entry of \( \mathrm{M} \), and \( \delta w_i \) is the \( i \)-th component of \( \delta w \).

\subsection*{Step 2: Expectation of \( \delta w_i \delta w_j \delta w_k \delta w_l \)}

Since \( \delta w \sim \mathcal{N}(\mathbf{0}, \mathbf{I}) \), the components \( \delta w_i \) are independent Gaussian random variables with mean 0 and variance 1. The expectation \( \mathbb{E}[\delta w_i \delta w_j \delta w_k \delta w_l] \) depends on the indices \( i, j, k, l \). Using properties of Gaussian random variables, we have:
\begin{equation}
\mathbb{E}[\delta w_i \delta w_j \delta w_k \delta w_l] =
\begin{cases}
1, & \text{if } i = j, \, k = l, \, i \neq k, \\
1, & \text{if } i = k, \, j = l, \, i \neq j, \\
1, & \text{if } i = l, \, j = k, \, i \neq j, \\
1, & \text{if } i = j = k = l, \\
0, & \text{otherwise}.
\end{cases}
\end{equation}

This result follows from the Wick formula for moments of Gaussian random variables.

\subsection*{Step 3: Substituting into the Expectation}

Return to the expectation of \( X^2 \):
\begin{equation}
\mathbb{E}[X^2] = \sum_{i,j} \sum_{k,l} \mathrm{M}_{ij} \mathrm{M}_{kl} \mathbb{E}[\delta w_i \delta w_j \delta w_k \delta w_l].
\end{equation}

Using the cases derived above, the non-zero contributions arise in the following scenarios:
\begin{itemize}
    \item \textbf{Case 1: \( i = j, k = l, i \neq k \):} The contribution is:
    \begin{equation}
    \sum_{i \neq k} \mathrm{M}_{ii} \mathrm{M}_{kk} = \mathrm{trace}(\mathrm{M})^2.
    \end{equation}

    \item \textbf{Case 2: \( i = k, j = l, i \neq j \):} The contribution is:
    \begin{equation}
    \sum_{i,j} \mathrm{M}_{ij}^2 = \mathrm{trace}(\mathrm{M}^2).
    \end{equation}

    \item \textbf{Case 3: \( i = l, j = k, i \neq j \):} This is identical to the second case, contributing:
    \begin{equation}
    \mathrm{trace}(\mathrm{M}^2).
    \end{equation}

    \item \textbf{Case 4: \( i = j = k = l \):} The contribution is:
    \begin{equation}
    \sum_i \mathrm{M}_{ii}^2 = \mathrm{trace}(\mathrm{M}^2).
    \end{equation}
\end{itemize}

Combining these terms, the total expectation is:
\begin{equation}
\mathbb{E}[X^2] = 2 \mathrm{trace}(\mathrm{M}^2) + \mathrm{trace}(\mathrm{M})^2.
\end{equation}

\subsection*{Step 4: Final Simplification}

Substituting back into the original equation:
\begin{equation}
\mathbb{E}\left[\left(\frac{1}{2} \delta w^\top \mathrm{M} \delta w\right)^2\right] = \frac{1}{4} \mathbb{E}[X^2] = \frac{1}{4} \left( 2 \mathrm{trace}(\mathrm{M}^2) + \mathrm{trace}(\mathrm{M})^2 \right).
\end{equation}

Simplify to obtain:
\begin{equation}
\mathbb{E}\left[\left(\frac{1}{2} \delta w^\top \mathrm{M} \delta w\right)^2\right] = \frac{1}{2} \mathrm{trace}(\mathrm{M}^2) + \frac{1}{4} \mathrm{trace}(\mathrm{M})^2.
\end{equation}

So, clearly the minimizer of $\Psi(\mathrm{X})$ is at $\mathrm{M} = 0$ or $\mathrm{X} = \mathrm{H}_r$. 

However it is the ideal case, where we do not approximate $\delta \mathcal{L}_r$. In our algorithm, we are minimizing and approximate objective $\mathrm{\Tilde{\Psi}}(\textrm{X})$. 
Also from the definition we can say $\Tilde{\mathrm{f}}_i(X) = \mathrm{f}_i(\textrm{X})+(\delta \mathcal{L}_r(w_i)-\delta \mathcal{L}_f(w_i))$. Since we assume that $|\delta \mathcal{L}_r(w_i) - \delta \mathcal{L}_f(w_i)| \leq \epsilon \ \forall i$, we can derive the following inequality:
\begin{align}
    & (\mathrm{f}_i(\textrm{X}) - \epsilon) \leq \Tilde{\mathrm{f}}_i(X) \leq (\mathrm{f}_i(\textrm{X}) + \epsilon) \\
\implies &  \frac{1}{m} \sum_{i=1}^m \left( \mathrm{f}_i(\textrm{X})-\epsilon)\right)^2 \leq \mathrm{\Tilde{\Psi}(X)} \leq \frac{1}{m} \sum_{i=1}^m \left( \mathrm{f}_i(\textrm{X}) + \epsilon)\right)^2
\end{align}

Now we know:

\begin{align}
     \frac{1}{m} \sum_{i=1}^m \left( \mathrm{f}_i(\textrm{X}) + \epsilon)\right)^2 &= \mathbb{E}_{\delta w \sim \mathcal{N}(\mathbf{0}, \mathbf{I})}\left [ (\frac{1}{2}\delta w^\top \mathrm{M} \delta w + \epsilon)^2\right] \\
     &= \mathbb{E}_{\delta w \sim \mathcal{N}(\mathbf{0}, \mathbf{I})}\left [ (\frac{1}{2}\delta w^\top \mathrm{M} \delta w )^2\right] + 2\epsilon \mathbb{E}_{\delta w \sim \mathcal{N}(\mathbf{0}, \mathbf{I})}\left [ (\frac{1}{2}\delta w^\top \mathrm{M} \delta w ) \right] + \epsilon^2 \\
     & \leq \frac{1}{2}\mathrm{trace}(\mathrm{M}^2) + \frac{1}{4}\mathrm{trace}(\mathrm{M})^2 + \frac{1}{2}\mathrm{trace}(2 \epsilon \mathrm{M}) \\
     &= \frac{1}{2}\mathrm{trace}(\mathrm{M}^2+ 2\epsilon \mathrm{M}) + \frac{1}{4}\mathrm{trace}(\mathrm{M})^2
\end{align}

Similarly expanding the lower bound also, we get the following inequality:

\begin{align}
    & \frac{1}{2}\mathrm{trace}(\mathrm{M}^2- 2\epsilon \mathrm{M}) + \frac{1}{4}\mathrm{trace}(\mathrm{M})^2 \leq \mathrm{\Tilde{\Psi}(X)} \leq \frac{1}{2}\mathrm{trace}(\mathrm{M}^2+ 2\epsilon \mathrm{M}) + \frac{1}{4}\mathrm{trace}(\mathrm{M})^2 \\
\end{align}

By seperately taking derivatives of the upper and lower bounds above, if we set it to $0$, we get the following bound on the minimizer $\mathrm{M}$.

\begin{equation}
    -\frac{2\epsilon}{(2+d)}\mathrm{I}_d \leq \mathrm{M} \leq \frac{2\epsilon}{(2+d)}\mathrm{I}_d
\end{equation}
where, $\mathrm{I}_d \in \mathbb{R}^{d \times d}$ is the identity matrix. \\

\noindent \textbf{Details:}\\

\noindent The objective function is defined as:
\begin{equation}
f(\mathrm{M}) = \frac{1}{2} \mathrm{Tr}(\mathrm{M}^2 + 2\epsilon \mathrm{M}) + \frac{1}{4} (\mathrm{Tr}(\mathrm{M}))^2,
\end{equation}
where \( \mathrm{M} \) is a \( d \times d \) matrix, and \( \epsilon \) is a scalar parameter. To find the optimal \( \mathrm{M} \), we compute the gradient of \( f(\mathrm{M}) \) with respect to \( \mathrm{M} \) and solve \( \nabla_{\mathrm{M}} f(\mathrm{M}) = 0 \). \\

\noindent The gradient of each term in \( f(\mathrm{M}) \) is as follows: \\
\begin{itemize}
    \item The gradient of \( \frac{1}{2} \mathrm{Tr}(\mathrm{M}^2) \) is:
    \begin{equation}
    \nabla_{\mathrm{M}} \left( \frac{1}{2} \mathrm{Tr}(\mathrm{M}^2) \right) = \mathrm{M}.
    \end{equation}
    
    \item The gradient of \( \epsilon \mathrm{Tr}(\mathrm{M}) \) is:
    \begin{equation}
    \nabla_{\mathrm{M}} \left( \epsilon \mathrm{Tr}(\mathrm{M}) \right) = \epsilon \mathrm{I}_d,
    \end{equation}
    where \( \mathrm{I}_d \) is the \( d \times d \) identity matrix. \\
    
    \item The gradient of \( \frac{1}{4} (\mathrm{Tr}(\mathrm{M}))^2 \) is:
    \begin{equation}
    \nabla_{\mathrm{M}} \left( \frac{1}{4} (\mathrm{Tr}(\mathrm{M}))^2 \right) = \frac{1}{2} \mathrm{Tr}(\mathrm{M}) \mathrm{I}_d.
    \end{equation}
\end{itemize}

Combining these terms, the total gradient is:
\begin{equation}
\nabla_{\mathrm{M}} f(\mathrm{M}) = \mathrm{M} + \epsilon \mathrm{I}_d + \frac{1}{2} \mathrm{Tr}(\mathrm{M}) \mathrm{I}_d.
\end{equation}

Setting \( \nabla_{\mathrm{M}} f(\mathrm{M}) = 0 \), we have:
\begin{equation}
\mathrm{M} + \epsilon \mathrm{I}_d + \frac{1}{2} \mathrm{Tr}(\mathrm{M}) \mathrm{I}_d = 0.
\end{equation}
\noindent Rearranging, this becomes:
\begin{equation}
\mathrm{M} = -\epsilon \mathrm{I}_d - \frac{1}{2} \mathrm{Tr}(\mathrm{M}) \mathrm{I}_d.
\end{equation}

Let \( \mathrm{Tr}(\mathrm{M}) = \tau \). Substituting \( \tau \) into the equation, we get:
\begin{equation}
\mathrm{M} = -\epsilon \mathrm{I}_d - \frac{1}{2} \tau \mathrm{I}_d.
\end{equation}

Taking the trace on both sides:
\begin{equation}
\tau = \mathrm{Tr}(\mathrm{M}) = \mathrm{Tr}\left(-\epsilon \mathrm{I}_d - \frac{1}{2} \tau \mathrm{I}_d\right).
\end{equation}

Since \( \mathrm{Tr}(\mathrm{I}_d) = d \) for a \( d \times d \) matrix:
\begin{equation}
\tau = -\epsilon d - \frac{1}{2} \tau d.
\end{equation}

Rearranging to isolate \( \tau \):
\begin{equation}
\tau \left( 1 + \frac{d}{2} \right) = -\epsilon d.
\end{equation}

Solving for \( \tau \):
\begin{equation}
\tau = \frac{-\epsilon d}{1 + \frac{d}{2}} = \frac{-2 \epsilon d}{2 + d}.
\end{equation}

Substituting \( \tau \) Back into \( \mathrm{M} \)
Substitute \( \tau = \frac{-2 \epsilon d}{2 + d} \) into \( \mathrm{M} = -\epsilon \mathrm{I}_d - \frac{1}{2} \tau \mathrm{I}_d \):
\begin{equation}
\mathrm{M} = -\epsilon \mathrm{I}_d - \frac{1}{2} \left( \frac{-2 \epsilon d}{2 + d} \right) \mathrm{I}_d.
\end{equation}

Simplify:
\begin{equation}
\mathrm{M} = -\epsilon \mathrm{I}_d + \frac{\epsilon d}{2 + d} \mathrm{I}_d.
\end{equation}

Combine terms:
\begin{equation}
\mathrm{M} = \left( -\epsilon + \frac{\epsilon d}{2 + d} \right) \mathrm{I}_d.
\end{equation}

Simplify further:
\begin{equation}
\mathrm{M} = -\epsilon \left( 1 - \frac{d}{2 + d} \right) \mathrm{I}_d = -\epsilon \left( \frac{2}{2 + d} \right) \mathrm{I}_d.
\end{equation}

The optimal \( \mathrm{M} \) is:
\begin{equation}
\mathrm{M} = -\frac{2 \epsilon}{2 + d} \mathrm{I}_d,
\end{equation}
where \( d \) is the dimension of the matrix \( \mathrm{M} \).\\
This inequality implies that the if the solution of optimization~\ref{opt:main_opt} is $\hat{\textrm{H}_r}$, then 

\begin{equation}
    \textrm{H}_r-\frac{2\epsilon}{(2+d)}\mathrm{I}_d \preceq \hat{\textrm{H}}_r \preceq \textrm{H}_r+\frac{2\epsilon}{(2+d)}\mathrm{I}_d
\end{equation}

\noindent As a result we can conclude:
\begin{equation}
   \|\hat{\textrm{H}}_r - \textrm{H}_r\| = \| \Delta \textrm{H}_r\|_F \leq \frac{2 \epsilon \|\mathrm{I}_d\|_F}{(2+d)}
\end{equation}

\noindent Since $\|\mathrm{I}_d\|_F = \sqrt{d}$, we conclude the proof.

\end{proof}

\section{Additional Experiments}

We conducted experiments on the CIFAR-10, CIFAR-100, StanfordDogs, and Caltech-256 datasets using our proposed method for both linear classifier and mixed linear network cases. For all experiments, 500 perturbations were applied. The “Performance Gap” row represents the difference in performance between the methods Unlearned (+) and Unlearned (-). Unlearned (+) refers to unlearning using the remaining data samples, while Unlearned (-), our proposed method, performs unlearning without relying on the remaining data samples.   

\subsection{Linear Classifier Experiments}

For the linear classifier experiments, a ResNet-18 model, pretrained on the ImageNet dataset and excluding the penultimate layer, was used to generate activations for our linear classifier. $10 \%$ of the data was selected to be forgotten. As shown in \cref{tab:supp-linear-class}, the Unlearned (-) method achieves unlearning performance that is significantly close to the Unlearned (+) method. This result demonstrates that our method can perform well even without access to the remaining dataset, provided that the theoretical assumptions hold.

\captionsetup[table]{skip=12pt} 
\begin{table}[H]
  \centering
  \renewcommand{\arraystretch}{1.4} 
  \setlength{\tabcolsep}{8pt} 
  \rowcolors{2}{gray!10}{white} 
  \begin{tabular}{l c c c c}
    \toprule
    \rowcolor{gray!20}
    \textbf{Method} & \textbf{Test Data} & \textbf{Remaining Data} & \textbf{Forget Data} & \textbf{MIA} \\
    \midrule
    \rowcolor{white}
    \multicolumn{5}{c}{\textbf{CIFAR-10}} \\ \midrule
    \rowcolor{cyan!10}
    Retrained & 72\% & 74\% & 72\% & 50\% \\
    Unlearned (+) & 70.3\% & 72.4\% & 70.2\% & 50.2\% \\
    Unlearned (-) & 70\% & 71\% & 68\% & 51.5\% \\
    \textbf{Performance Gap} & \textbf{0.3\%} & \textbf{1.4\%} & \textbf{1.8\%} & \textbf{1.3\%} \\
    \midrule
    \rowcolor{white}
    \multicolumn{5}{c}{\textbf{CIFAR-100}} \\ \midrule
    \rowcolor{cyan!10}
    Retrained & 56.0\% & 61.4\% & 56.2\% & 50.4\% \\
    Unlearned (+) & 49.8\% & 59.7\% & 48.7\% & 51.4\% \\
    Unlearned (-) & 51.6\% & 59.6\% & 49.2\% & 51.8\% \\
    \textbf{Performance Gap} & \textbf{1.8\%} & \textbf{0.1\%} & \textbf{0.5\%} & \textbf{0.4\%} \\
    \midrule
    \rowcolor{white}
    \multicolumn{5}{c}{\textbf{StanfordDogs}} \\ \midrule
    \rowcolor{cyan!10}
    Retrained & 59.3\% & 67.8\% & 60.2\% & 50.2\% \\
    Unlearned (+) & 54.6\% & 65.1\% & 53.5\% & 51.7\% \\
    Unlearned (-) & 55.0\% & 68.3\% & 54.5\% & 50.9\% \\
    \textbf{Performance Gap} & \textbf{0.4\%} & \textbf{3.2\%} & \textbf{1.0\%} & \textbf{1.2\%} \\
    \midrule
    \rowcolor{white}
    \multicolumn{5}{c}{\textbf{Caltech-256}} \\ \midrule
    \rowcolor{cyan!10}
    Retrained & 54.4\% & 61.1\% & 54.8\% & 50\% \\
    Unlearned (+) & 47.6\% & 57.3\% & 45.0\% & 51.3\% \\
    Unlearned (-) & 49.4\% & 57.9\% & 48.4\% & 50.6\% \\
    \textbf{Performance Gap} & \textbf{1.8\%} & \textbf{0.6\%} & \textbf{3.4\%} & \textbf{0.7\%} \\
    \bottomrule
  \end{tabular}
  \caption{Linear classifier experiments on CIFAR-10, CIFAR-100, StanfordDogs, and Caltech-256 datasets. A ResNet-18 model, pretrained on ImageNet and excluding the penultimate layer, was used to generate activations for the linear classifier. $10 \%$ of the data was selected for unlearning. The "Performance Gap" row indicates the difference between the methods Unlearned (+) and Unlearned (-), where Unlearned (+) utilizes the remaining data samples, and Unlearned (-) (our proposed method) operates without access to the remaining data samples.}
  \label{tab:supp-linear-class}
\end{table}

\subsection{Mixed-Linear Network Experiments}

For the mixed linear network experiments, a ResNet-18 model pretrained on the ImageNet dataset was used as the base model. The last few layers were linearized for training on the datasets. $15 \%$ of the data was selected for unlearning. Unlearned (+) refers to the unlearning process that utilizes the remaining data samples, while Unlearned (-) (our proposed method) performs unlearning without access to the remaining data samples. As shown in \cref{tab:supp-mix-lin}, the Unlearned (-) and Unlearned (+) methods exhibit very similar performance. This result demonstrates that our method (Unlearned (-)) can achieve significantly strong results even without access to the remaining samples for the mixed linear network case as well.

\captionsetup[table]{skip=12pt} 
\begin{table}[H]
  \centering
  \renewcommand{\arraystretch}{1.4} 
  \setlength{\tabcolsep}{8pt} 
  \rowcolors{2}{gray!10}{white} 
  \begin{tabular}{l c c c c}
    \toprule
    \rowcolor{gray!20}
    \textbf{Method} & \textbf{Test Data} & \textbf{Remaining Data} & \textbf{Forget Data} & \textbf{MIA} \\
    \midrule
    \rowcolor{white}
    \multicolumn{5}{c}{\textbf{CIFAR-10}} \\ \midrule
    \rowcolor{cyan!10}
    Retrained & 86.3\% & 93.6\% & 87.7\% & 50.2\% \\
    Unlearned (+) & 85.6\% & 91.9\% & 85.5\% & 50.0\% \\
    Unlearned (-) & 84.5\% & 93.5\% & 86.7\% & 51.2\% \\
    \textbf{Performance Gap} & \textbf{1.1\%} & \textbf{1.6\%} & \textbf{1.2\%} & \textbf{1.2\%} \\
    \midrule
    \rowcolor{white}
    \multicolumn{5}{c}{\textbf{CIFAR-100}} \\ \midrule
    \rowcolor{cyan!10}
    Retrained & 63.1\% & 68.9\% & 63.3\% & 50.2\% \\
    Unlearned (+) & 61.9\% & 67.9\% & 62.2\% & 50.7\% \\
    Unlearned (-) & 61.4\% & 70.1\% & 62.2\% & 51.7\% \\
    \textbf{Performance Gap} & \textbf{0.5\%} & \textbf{2.2\%} & \textbf{0.0\%} & \textbf{1.0\%} \\
    \midrule
    \rowcolor{white}
    \multicolumn{5}{c}{\textbf{StanfordDogs}} \\ \midrule
    \rowcolor{cyan!10}
    Retrained & 73.6\% & 76.8\% & 72.1\% & 50.6\% \\
    Unlearned (+) & 69.0\% & 76.1\% & 70.8\% & 50.2\% \\
    Unlearned (-) & 70.1\% & 77.6\% & 71.2\% & 51.3\% \\
    \textbf{Performance Gap} & \textbf{1.1\%} & \textbf{1.5\%} & \textbf{0.6\%} & \textbf{1.1\%} \\
    \midrule
    \rowcolor{white}
    \multicolumn{5}{c}{\textbf{Caltech-256}} \\ \midrule
    \rowcolor{cyan!10}
    Retrained & 60.3\% & 66.9\% & 60.5\% & 50.0\% \\
    Unlearned (+) & 61.3\% & 65.6\% & 61.8\% & 49.8\% \\
    Unlearned (-) & 58.4\% & 66.2\% & 60.7\% & 51.0\% \\
    \textbf{Performance Gap} & \textbf{2.9\%} & \textbf{1.6\%} & \textbf{1.1\%} & \textbf{1.2\%} \\
    \bottomrule
  \end{tabular}
   \caption{Mixed linear network experiments on CIFAR-10, CIFAR-100, StanfordDogs, and Caltech-256 datasets. A ResNet-18 model, pretrained on the ImageNet dataset, was used as the base model. The last few layers were linearized for training with the datasets. $15 \%$ of the data was selected for unlearning. The "Performance Gap" row indicates the difference between the methods Unlearned (+) and Unlearned (-). Unlearned (+) performs unlearning using the remaining data samples, while Unlearned (-), our proposed method, achieves competitive results even without access to the remaining data samples.}
  \label{tab:supp-mix-lin}
\end{table}

\end{document}